\documentclass{llncs}
\usepackage[T1]{fontenc}
\usepackage{graphicx}
\usepackage{color}
\usepackage{xcolor}
\usepackage{tikz}
\usepackage{color}
\usepackage{cite}
\usepackage{amsmath,amssymb,amsfonts}
\usepackage[linesnumbered,ruled,vlined,noend]{algorithm2e}
\usepackage{booktabs}

\definecolor{drawblue}{RGB}{0, 110, 255}
\definecolor{drawgreen}{RGB}{0, 180, 0}
\definecolor{drawmagenta}{RGB}{209, 0, 163}

\newcommand{\BO}[1]{O\left(#1\right)}

\newcommand{\BT}[1]{\Theta\left(#1\right)}

\DeclareMathOperator*{\argmax}{argmax} %
\DeclareMathOperator*{\argmin}{argmin} %
\DeclareMathOperator*{\NN}{NN} 

\newtheorem{fact}{Fact}

\begin{document}
\pagestyle{plain}

\title{Streaming algorithms for robust max-min diversification}
%
%
\author{Andrea Pietracaprina \and Geppino Pucci \and Stefano Zanon}
\authorrunning{A. Pietracaprina et al.}
%
\institute{Department of Information Engineering, University of Padova, Italy
\email{\footnotesize\{andrea.pietracaprina,geppino.pucci\}@unipd.it, \ stefano.zanon.3@studenti.unipd.it}}
\maketitle   
\thispagestyle{plain}
\begin{abstract}
Given a set of $n$ points $X$ in a metric space and an integer $k$, max-min diversification aims to select
$k$ points of $X$  maximizing their minimum pairwise distance. This 
objective function is however highly vulnerable to noisy points. 
In \cite{AmagataAAAI23}, a robust formulation is proposed which addresses this vulnerability by excluding solutions containing any of $z$ outliers, defined as the $z$ points in $X$ with the largest nearest-neighbor distances. That paper also presents a coreset-based streaming algorithm for the new formulation, based on a suitable inlier-outlier separation assumption. However, we identify three shortcomings in the algorithm
by \cite{AmagataAAAI23}: 
its coreset construction requires an offline computation over $X$, which needs memory linear in $n$, in stark contrast with the typical goals of stream processing; the one-pass procedure used to extract the solution from the coreset may return fewer than $k$ points (hence, an unfeasible solution) because it permanently discards points too far from the current solution; and its outlier-exclusion guarantee is only probabilistic and weakens as the coreset size shrinks. In contrast, we present a deterministic coreset-based algorithm that, under a natural inlier–outlier separation assumption (similar to the one used in \cite{AmagataAAAI23}), returns exactly $k$ inliers which are a $(2+\varepsilon)$-approximate solution, for any  $\varepsilon>0$, thus only $\varepsilon$ above the best polynomial-time sequential approximation, even without outliers. Its one-pass streaming implementation adapts obliviously to the dataset's doubling dimension $D$ and, for wide ranges of $k$, $z$, $\varepsilon$, and $D$, it
uses memory independent of $n$. For sufficiently long streams, its amortized update time is proportional to the coreset size, thus also independent of $n$.
\end{abstract}
\section{Introduction} \label{sec-intro}

Given a set $X$ of $n$ points from a metric space and an integer $k$, the \emph{maximum diversification} (also known as \emph{diversity maximization}) problem requires to select $k$ points of $X$ maximizing some measure of how different the selected points are from one another. Maximum diversification is a fundamental primitive in the analysis of large datasets, as it returns a small subset of the input points which preserves its variety, and it has applications in web search, recommender systems, databases and data mining \cite{CeccarelloVLDB17,AmagataAAAI23}. Among the measures proposed in the literature (see \cite{IndykMMM14} for an exhaustive list), one of the most popular is the \emph{max-min} (also known as \emph{remote-edge}) measure, which we adopt in this work, and states that the diversity of a set $S$ is the minimum distance between any two distinct points of $S$. The resulting \emph{max-min diversification} problem is NP-hard, and the simple \emph{farthest-first} heuristic  of \cite{Gonzalez85} (widely known as \emph{GMM}) returns a $2$-approximate solution in $O(kn)$ time, which is the best that can be achieved in polynomial time unless $P = NP$ \cite{Ravi94}.

Since the objective function is a minimum of distances, the returned solution may be very sensitive to the presence of noise: for instance, spurious data points lying very far away from all other points in $X$, may unduly increase the value of the objective, and GMM, which iteratively selects the next point as the farthest from the previously selected ones, tends to return these noisy and hardly representative points of the dataset \cite{AmagataAAAI23}. This issue calls for a \emph{robust} formulation of the problem, where a fixed subset of $z$ points of $X$ are regarded as \emph{outliers} and are forbidden to appear in \emph{any} solution. 
The \emph{max-min diversification} problem \emph{with outliers},
introduced by \cite{AmagataAAAI23} and considered in our work, 
requires, given a dataset $X$ of $n$ points, and two integer parameters
$k$ and $z$, with $k+z \leq n$, to determine the $k$ points that optimize the max-min measure, among the points of $X \backslash X_{out}$, where $X_{out}$ are the $z$ points of $X$ farthest from their respective nearest neighbors. We note that this formulation is different from the one used for the robust version of the well-known $k$-center problem, where each solution (a set of $k$ points) induces a different set of $z$ outliers, defined as the $z$ farthest points from that solution, and whose exclusion thus improves the objective function the most. Instead, when outliers are a fixed subset of the input, they must  be identified by some criterion which does not depend on the objective function. 

When the dataset is too large to be stored in main memory, it is natural to process it as a stream, that is, by a single processor which reads the input in one sequential pass using a working memory much smaller than $n$. A very effective strategy in this scenario is the extraction of a \emph{$c$-coreset} \cite{Agarwal04}, a small subset of the input which contains a $c$-approximation of the optimal solution for the whole dataset, for some $c>1$, so that the execution of an expensive sequential algorithm for the problem can run on the coreset rather than on the entire dataset. As shown in \cite{CeccarelloVLDB17,CeccarelloPP19,CeccarelloPP20,CeccarelloPPS23,CeccarelloPPV25} for a host of problems, the quality of the coresets that can be extracted, and hence the final approximation, is often regulated by the doubling dimension $D$ of the metric space  underlying the input instance \cite{Gupta03}.

\subsection{Previous work}

Max-min diversification on datasets without outliers has been studied extensively. Besides the sequential results mentioned above, coreset-based MapReduce and Streaming algorithms have been presented in \cite{CeccarelloVLDB17}, for several diversity measures, in spaces of bounded doubling dimension $D$. Specifically, the algorithms in  \cite{CeccarelloVLDB17} attain an $(\alpha+\varepsilon)$-approximation, for any fixed $\varepsilon > 0$, where $\alpha$ is the best ratio featured by a polynomial-time sequential algorithm for the same measure, hence $(2+\varepsilon)$ in the case of max-min diversification. The same approach was later extended to the solution of diversity maximization under matroid constraints \cite{CeccarelloPP20}.

\sloppy For the $k$-center problem, the previously mentioned formulation with $z$ outliers has received considerable attention. In \cite{McCutchenK08}, the authors give a $1$-pass Streaming algorithm attaining a $(2+\varepsilon)$-approximation with working memory $O(k\varepsilon^{-1}\log \varepsilon^{-1})$ for the problem without outliers, together with a $(4+\varepsilon)$-approximation with working memory $O(kz\varepsilon^{-1})$ for the problem with $z$ outliers. Coreset-based MapReduce and Streaming algorithms, whose approximation ratio is a mere additive term $\varepsilon$ away from the best sequential one, are given in \cite{CeccarelloPP19}, and, under matroid constraints, in \cite{CeccarelloPPS23}. Our streaming algorithm for max-min diversification with $z$ outliers builds on the coreset construction employed in this  latter work.

To the best of our knowledge, diversity maximization in the presence of outliers has been considered only in \cite{AmagataAAAI23}, which introduced the formulation described above and proposed two algorithms. The first algorithm, \textsc{Greedy}, runs GMM to select $k+z$ points, which, as the author proves, contain all the outliers. Then, by computing the distances to their respective nearest neighbors for these $k+z$ points, the outliers are identified as the $z$ points whose computed distances are the largest, and a $2$-approximate solution is found by running GMM for $k$ points on the whole dataset cleaned of the just identified outliers. The second algorithm, \textsc{Coreset}, runs a streaming algorithm on a coreset $C$ of $\BO{k+z}$ points extracted with GMM. The streaming strategy used on $C$ works as follows. Let $S'$ be a solution of $k+z$ points computed with GMM, where $div(S')$ denotes its diversity measure, that is, the minimum distance between any two distinct points of $S'$. It is shown that the diversity measure of the optimal solution, denoted as $\ell_k^*$, lies in the interval $[div(S'),2div(S')]$. For a geometric sequence of guesses
$\hat{\ell} \in [div(S'),2div(S')]$, a solution $S$ is then computed by taking a random point of $C$ and then processing each $x \in C$ in a streaming fashion, adding $x$ to $S$ only if $d(x,S) \in [\hat{\ell}/2,\hat{\ell}]$. Finally, the solution with $k$ points and maximum diversity, among the ones computed for the various guesses, is returned. The author claims that the processing of $C$ returns, with constant probability, a $6(1+\varepsilon)$-approximate solution with no outliers, in time independent of $n$.  Both of the algorithms  in
\cite{AmagataAAAI23} crucially rely on the assumption that the outliers are the $z$ farthest points from the respective nearest neighbors, and that the distance between any outlier and its nearest neighbor is at least $\alpha \ell_k^*$, where $\alpha > 1$ and $\ell_k^*$ is the value of the objective function of the optimal solution excluding outliers. 

Our work stems from a critique of the streaming approach of \cite{AmagataAAAI23}, which suffers from a number of drawbacks and inaccuracies, listed below.
\begin{itemize}
\item The claim that their two-phase approach yields a streaming algorithm is rather questionable. Indeed, the  coreset is built in the first phase \emph{offline}, through a full pass over $X$, which requires \emph{working memory linear in $n$}, thus defying the purpose of a streaming algorithm, which is primarily meant for datasets too large to fit in memory. Moreover, since a coreset is by construction small enough to reside entirely in memory, relatively little is gained by processing it as a  stream in the second phase.

\item 
The ``scan-once'' behavior used in the second phase of \textsc{Coreset} to process the coreset $C$ may fail to return a solution with $k$ points. Consider a dataset $X$ consisting of two pairs of points, with intra-pair distance $1$ and inter-pair distance $\geq 3$, and set $k=4$ and $z=0$ (see Fig.~\ref{fig:schema}). The optimal $k$-diversity of the instance is clearly $\ell_k^* = 1$. In this case, the coreset $C$ of $k+z=4$ points computed by Algorithm \textsc{Coreset} through GMM consists of all points of $X$, moreover $S' = X$, hence $div(S')=1=\ell_k^*$. Now, for every guess $\hat{\ell} \in [div(S'),2div(S')]=[1,2]$, the algorithm initializes a solution $S$ with a random point $p \in C$, and scans each remaining point $x$ of $C$, adding it to $S$ only if $d(x,S) \in [\hat{\ell}/2,\hat{\ell}] \subseteq [\ell_k^*/2,2\ell_k^*] = [1/2,2]$. Therefore, only the other point belonging to the same pair as $p$ can be added to $S$, and, at the end $S$ will contain only $2< k$ points. In this case, the algorithm fails to return a valid solution. In general, since a point $x$ is discarded if, when $x$ is processed, it is found to be too far from the current partial solution $S$, the algorithm may fail to return a solution, depending on the order in which the points are processed. If points far from $S$ occur before closer ones, then the algorithm will discard them permanently, and $S$ will never reach the target size $k$.
\begin{figure}[htb]
    \centering
    \resizebox{0.7\textwidth}{!}{%
    \begin{tikzpicture}[scale=0.6, >=stealth, line cap=round, line join=round]

        \begin{scope}
            \path[clip] (-10, -11) rectangle (15, 11)
                                       (0, -2) circle (2)
                                       (0, 2) circle (2);

            \fill[gray, opacity=0.15] (0, -2) circle (8) (0, 2) circle (8);
        \end{scope}

        \draw[thick] (0, -2) circle (2);
        \draw[thick] (0, -2) circle (8);
        \fill (0, -2) circle (4pt);

        \draw[thick] (0, 2) circle (2);
        \draw[thick] (0, 2) circle (8);
        \fill (0, 2) circle (4pt);

        \draw[drawgreen, thick] (0, -2) -- (0, 2)
            node[pos=0.3, right, xshift=2pt, fill=white, inner sep=1pt, text=drawgreen, font=\Large] {$\ell_k^*$};

        \draw[drawblue, thick] (0, 2) -- (8, 2)
            node[midway, above, text=drawblue, font=\Large] {$2\ell_k^*$};

        \draw[drawmagenta, thick] (0, 2) -- +(120:2)
            node[left, xshift=-1pt, yshift=2pt, text=drawmagenta, font=\large] {$\ell_k^*/2$};

        \fill (12, 2) circle (4pt);
        \fill (12, -2) circle (4pt);

    \end{tikzpicture}%
    }
    \caption{Visual execution of the second phase of the \textsc{Coreset} algorithm presented in \cite{AmagataAAAI23}.}
    \label{fig:schema}
\end{figure}
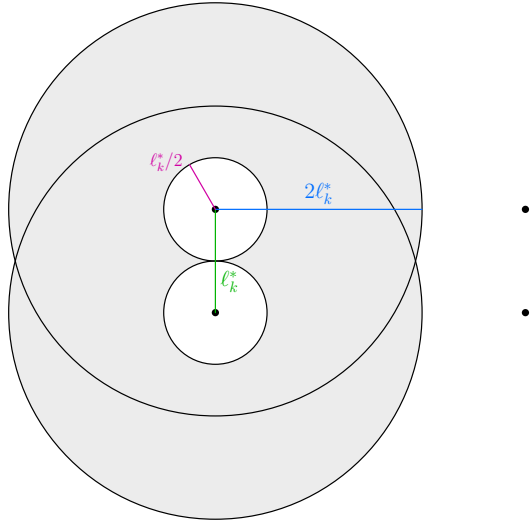
\item As explicitly stated in \cite{AmagataAAAI23}, the guarantee that the returned solution contains no outliers is only probabilistic, and it degrades as the size of the coreset decreases, since the success probability is $1 - z/|C|$. Thus, a  constant success probability $p$ calls for a coreset of at least $|C| = z/(1-p)$ points, a requirement in conflict with the goal of maintaining the coreset small.
\end{itemize}

\subsection{Our contribution}
We present coreset-based algorithms for max-min diversification  with outliers, which do not suffer from the limitations of the results by \cite{AmagataAAAI23} described above. Our algorithms attain high accuracy and performance for spaces of low doubling dimension $D$, a popular notion of dimensionality for general metrics.  
Our results hold under a separation assumption between inliers and outliers, expressed in terms of two suitably separated constants $\alpha$ and $\beta$, with $\alpha >1$ and $\alpha> \beta> 0$, so that the nearest neighbor of an outlier is at distance at least $\alpha \ell^*_k$, while the nearest neighbor of an inlier is at distance at most $\beta \ell^*_k$, where $\ell^*_k$ is the value of the objective function of the optimal solution, excluding the outliers. In \cite{AmagataAAAI23}, the same assumption was made, but with
$\beta=\alpha$. Making $\beta$ a distinct parameter from $\alpha$, which 
essentially imposes a sharper (yet reasonable) separation between outliers and inliers, allows us to identify the outliers exactly and deterministically. 

More specifically, we provide:

\begin{itemize}
\item A sequential algorithm which extracts a coreset $T$ by running GMM on $X$, identifies the outliers within the coreset through their nearest neighbor distances, and computes the final solution on the remaining points. While \cite{AmagataAAAI23} had already observed that $k+z$ points are enough for $T$ to contain all the outliers and to be a $2$-coreset\footnote{More precisely, \cite{AmagataAAAI23} shows that a solution $S'$ obtained with GMM$(X, k+z)$ has diversity at least half of the optimal one. This, together with the fact that $S'$ contains all the outliers, equivalently $S'$ contains $k$ inliers, is sufficient to prove that $S'$ is a $2$-coreset.}, we propose a simpler strategy to identify the outliers. Moreover, and more importantly, we show that $k(8/\varepsilon')^D + z$ points, with $\varepsilon' = \varepsilon / (1+\varepsilon)$, make $T$ a $(1+\varepsilon)$-coreset. For constant values of $k$ and $z$ and datasets of small (constant) doubling dimension, this coreset-based approach is able to yield $(1+\varepsilon)$ approximate solutions in linear time. 
\item A deterministic $1$-pass Streaming algorithm which returns a $(2+\varepsilon)$-approximate solution of $k$ points, containing no outliers, using a working memory of size $O((k+z)(c/\varepsilon)^{D})$, for a suitable constant $c$, hence independent of $n$ and substantially sublinear for wide ranges of $k$, $z$, $\varepsilon$ and $D$. For sufficiently long streams, its amortized update time is proportional to the size of the coreset. 
\end{itemize}

In both cases the approximation ratio is a mere additive term $\varepsilon$ away from the best ratio attainable in polynomial time even in the absence of outliers, and it improves substantially over the claimed $6(1+\varepsilon)$ ratio of \cite{AmagataAAAI23}, which was moreover only guaranteed with constant probability and was plagued with the issues amply discussed 
before.

As is the case for \cite{CeccarelloVLDB17,CeccarelloPP19,CeccarelloPPS23}, our algorithms adapt obliviously to the doubling dimension $D$, in the sense that $D$ is used in the analysis but is not needed by the algorithms.

\medskip
\noindent
\textbf{Organization of the paper.} The rest of the paper is organized as follows. Section \ref{sec-prelim} introduces the notation, the formal definition of the problem, and the concepts on which our algorithms are based. Section \ref{sec-algorithm} presents and analyzes our sequential algorithm, while Section \ref{streaming-algorithm} presents and analyzes the streaming one.

\section{Preliminaries} \label{sec-prelim}

In this section, we introduce the fundamental notation and formally define the problem of max-min diversification in the presence of outliers. We also review the geometric and algorithmic concepts that lay the foundation for our work. 
Throughout the paper, we let $X$ denote the input dataset
consisting of $n$ points from a metric space with distance function $d(\cdot,\cdot)$. For any point $x \in X$, and subset $Y \subseteq X$, we use $d(x,Y)$ as a shorthand notation to denote $\min_{y \in Y} d(x,y)$.

\subsection{Problem definition} 
For a given subset $S \subseteq X$, 
let the \emph{diversity} of $S$, denoted as $div(S)$, be the minimum pairwise distance between its points\footnote{In the literature, this diversity measure is referred to as \emph{remote-edge diversity}, and it is one of several notions of diversity which are studied \cite{IndykMMM14}.}, that is
\[
div(S) = \min \{d(x,y) \; : \; x,y \in S, x \neq y\}.
\]
The \emph{max-min diversification problem} 
requires that, given an input set $X$ of $n$ points and an integer $k \in [2,n]$, the subset of $k$ points of maximum diversity be computed, defined as
\[
S^* = \argmax_{S \subseteq X, |S| = k}{div(S)}.
\]
We will use $div_k(X)$ to denote the diversity of the
optimal solution to the problem for the set $X$, that is, $div_k(X) = div(S^*)$.

Real-world datasets often contain noisy data points which may
obfuscate the outcome of the analysis, and, due to the formulation of its objective function, max-min diversification is particularly vulnerable to the presence of such noise.  To deal with this scenario, the following more general formulation of the  max-min diversification problem has been proposed by \cite{AmagataAAAI23}. Suppose  that the input set $X$ contains a subset $X_{out}$ of $z$ \emph{outliers}, namely noisy points to be ignored by the analysis, and let $X_{in} = X \backslash X_{out}$ denote the remaining points, which we refer to as \emph{inliers}. For a point
$x \in X$, let $\NN(x)$ denote its nearest neighbor in $X$, that is, 
\[
\NN(x) = \argmin_{x' \in X \setminus \{x\}}{d(x, x')}
\]
The notion of outliers used in our paper is the one introduced by \cite{AmagataAAAI23}, and states that $X_{out}$ contains the $z$ points of $X$ with largest distance from their nearest neighbor. Given an input set $X = X_{in} \cup X_{out}$ of $n$ points with $|X_{out}|=z$, and an integer $k \in [2,|X_{in}|]$, the \emph{max-min diversification problem with outliers} 
requires to return the subset 
\[
S^* = \argmax_{S \subseteq X_{in}, |S| = k}{div(S)}.
\]
Let  $\ell_k^* = div(S^{*}) = div_k(X_{in})$.  Our algorithms assume that the input dataset satisfies the following \emph{separation assumption}: there exist two (suitably separated) constants $\alpha>1$ and $\beta>0$ with $\alpha > \beta$ such that

\[
\begin{array}{llll}
d(x, \NN(x)) & \ge & \alpha \ell_k^* &
\;\;\; \forall x \in X_{out}, \\
d(x,\NN(x)) & \le  & \beta \ell_k^* &
\;\;\; \forall x \in X_{in}, 
\end{array}
\]
This assumption slightly generalizes the one used in \cite{AmagataAAAI23},
where $\beta = \alpha$. As observed in \cite{AmagataAAAI23},
this type of separation has been considered in previous works 
on distance-based and nearest-neighbor based outlier detection 
\cite{AmagataOH21,AmagataOH22,GuAR19}.

\subsection{$\gamma$-Coreset}
The concept of a coreset captures the idea of a small set of points that approximates some property of a larger set \cite{Agarwal04,CeccarelloVLDB17}. 
\begin{definition}
Given a dataset $X = X_{in} \cup X_{out}$ with outliers, consider a subset $T \subseteq X$, and suppose that it contains at least $k$ inliers of $X$. Let $T_{in} = T \cap X_{in}$. For a value $\gamma \geq 1$, we say that $T$ is a $\gamma$-coreset for max-min diversification with outliers on $X$ if
\[
div_k(T_{in}) \ge \frac{1}{\gamma} div_k(X_{in}).
\]
\end{definition}
The definition immediately implies that, given a $\gamma$-coreset $T$, and assuming that $T_{in}$ can be easily retrieved from $T$, if a $c$-approximation algorithm for max-min diversification is run on $T_{in}$, the returned solution $S$, 
is a $c \cdot \gamma$-approximation to max-min diversification with outliers on $X$, that is $div(S) \geq (1/(c \cdot \gamma)) div_k(X_{in})$.

Moreover, the following fact straightforwardly derives from the above definition and from the aforementioned separation assumption.
\begin{fact} \label{fact-outliers}
Let $X = X_{in} \cup X_{out}$ be a dataset with outliers, and let $T \subseteq X$ be a subset such that 
\[
\max_{x \in X} d(x,T) < \alpha \ell_k^*.
\]
Then $X_{out} \subseteq T$.
\end{fact}
\subsection{Doubling Dimension}
To analyze theoretical guarantees in general metric spaces, we rely on the doubling dimension of the input. The doubling dimension of a set $X$ is the smallest value $D$ such that any ball of radius $r$ can be entirely covered by at most $2^{D}$ balls of radius $r/2$ \cite{Gupta03}.
As an immediate consequence, for any $\varepsilon \in (0,1]$, any ball of radius $r$ can be covered by at most 
$2^{\lceil \log_2 (1/\varepsilon) \rceil \cdot D} < (2/\varepsilon)^D$ balls of radius at most $\varepsilon r$.

\subsection{$k$-Center Problem}
\label{sec:kcenter}
As well known from previous works \cite{Ravi94,CeccarelloVLDB17}, max-min diversification is tightly related to the \emph{$k$-center problem} which, given a set $X$ of $n$ points and an integer $k \leq n$, requires to find a subset $S \subseteq X$ of $k$ centers, which minimizes $\max_{x \in X} d(x,S)$. The solution $S$ induces a partition of $X$ into $k$ clusters, where each point of $X$ is assigned to the cluster of its closest center (with ties broken arbitrarily). The value $\max_{x \in X} d(x,S)$ is referred to as the \emph{radius} of the induced clustering. 
Throughout the paper, we will use $r^*_k(X)$, for any given  
integer $k$, to denote the \emph{optimal radius} for the $k$-center problem on $X$, namely
\[
r^*_k(X) = \min_{S \subseteq X, |S| = k}
{\max_{x \in X} d(x,S)}.
\]

Let $r_k^*(X)$ be the radius of the clustering induced by the optimal solution to the $k$-center problem for $X$. As stated in \cite[Fact 1]{CeccarelloVLDB17}, the following relation, called \emph{anticover property}, holds:
\begin{equation} \label{eq-anticover}
r_k^*(X) \leq div_k(X).
\end{equation}

\subsubsection{Sequential GMM Algorithm}
The \emph{Greedy Max-Min} (GMM) algorithm \cite{Gonzalez85} is a popular 2-approximation algorithm for $k$-center, and, as shown in  \cite{Ravi94}, the
returned solution is also a $2$-approximation to the max-min diversification problem without outliers. The algorithm initializes the solution $S$ by picking an arbitrary point $x^{(1)}$ in $X$, and then 
iteratively adds to $S$ the point of $X$ farthest from $S$, until $|S| = k$. (See Algorithm 1 for the pseudocode). 

\medskip

\DontPrintSemicolon
\begin{algorithm}[H]
    \caption{GMM($X, k$)}
    $S \gets \{x^{(1)}\}$, with $x^{(1)}$ arbitrary point in $X$\;
    \For{$2 \leq i \leq k$}{
        $x^{(i)} \gets \argmax\limits_{y \in X \setminus S} d(y, S)$\;
        $S \gets S \cup \{x^{(i)}\}$\;
    }
    \Return{$S$}\;
\end{algorithm}

\medskip
The algorithm can be easily implemented in $\BO{nk}$ time and, within the same time bound, it is easy to maintain, for every $x \in S$, the distance from its nearest neighbor (i.e., $d(x,\NN(x))$).
For $1 \leq i \leq k$, let $S^{(i)} = \{x^{(1)}, x^{(2)}, \dots, x^{(i)}\}$
be the set $S$ at the end of the $i$-th iteration of GMM. Thus, $S^{(k)}$ is the final set $S$ returned by the algorithm. Let also 
$x^{(k+1)}$ denote the farthest point of $X$ from $S^{(k)}$ (i.e., the point that would be added to $S$ in a further iteration). For $2 \leq i \leq k+1$ define
\[
r_i = d(x^{(i)},S^{(i-1)}).
\]
The properties stated in the following fact are easily established
\cite{Ravi94,Gonzalez85}.
\begin{fact} \label{fact-gmm}
We have:
\begin{enumerate}
\item 
$r_i = div(S^{(i)})$, for $2 \le i \le k+1$.
\item 
The dataset $X$ can be covered with $i-1$ balls  centered in the points in $S^{(i-1)}$, of radius 
\[
r_i \leq r_{i-1} = div(S^{(i-1)}), \text{ for } 3 \le i \le k+1.
\]
\end{enumerate}
\end{fact}

\subsubsection{Streaming Scaling Algorithm}

In \cite{McCutchenK08}, the authors present a streaming algorithm for the $k$-center problem. For a fixed, user-defined, accuracy parameter
$\delta \in (0,1)$, at any time the algorithm maintains a pair $(S,r)$ such that 
\begin{itemize}
\item
$S$  is a set of $k$ centers providing a $(2+\delta)$-approximate solution to $k$-center for the prefix $X$ of the stream  
arrived up to that time.
\item 
$r$ is an estimate of the optimal radius such that $r \leq (2+\delta)r_k^*(X)$.
\end{itemize}

The algorithm runs $m = \BT{(1/\delta) \log (1/\delta)}$ parallel instances which, at any time, maintain $m$ different solution pairs $(S_i,r_i)$, with $1 \leq i \leq m$. All instances are initialized with the
first $k+1$ points of the stream. After the initialization, each instance $i$ evolves independently and, when a new point is processed, the set $S_i$ is suitably updated and the value of the estimated radius $r_i$ either remains unchanged, or it grows by a factor $\BT{1/\delta}$. At any time, the overall solution $(S,r)$ that the algorithm maintains is the best among the solutions $(S_i,r_i)$ of the individual instances (i.e., the one with smallest $r_i$). For what concerns the time and space complexity of the algorithm, the authors show that 
\begin{itemize}
\item 
At any time, the algorithm requires a \emph{working memory} of size 
\[
\BO{mk} = \BO{(1/\delta) \log (1/\delta)k}.
\]
\item 
The \emph{initialization time} is $\BO{k^2}$. 
\item 
The \emph{worst-case update time} to handle the arrival of a new point, after the $(k+1)$-st one,  
is 
\[
\BO{mk^2} = \BO{(1/\delta) \log (1/\delta)k^2}.
\] 
\item 
The \emph{amortized update time} for a stream $X$ of $n$ points is 
\[
\BO{m\left(k + (k^2 \log_{1/\delta} \Delta)/n\right)} = \BO{(1/\delta)\left(k\log (1/\delta) + (k^2 \log \Delta)/n\right)},
\]
\end{itemize}
where $\Delta$ is the aspect ratio of $X$, namely the ratio between the maximum and the minimum pairwise distances, $d_{max}$ and $d_{min}$ respectively, between points of $X$, that is 
$\Delta = d_{max}/d_{min}$.

In the paper, we refer to the above algorithm as the 
\emph{$\delta$-Scaling Algorithm}.

\section{Coreset-based sequential algorithm} \label{sec-algorithm}

This section presents a simple coreset-based sequential strategy for max-min diversification with outliers, which will be at the base of the streaming algorithm presented in the next section. We will also show that it may efficiently compute very accurate solutions for certain ranges of the parameters. 

The following lemma, proved in \cite[Lemma 1]{AmagataAAAI23}, states an important property regarding GMM.
\begin{lemma} \label{lem:outliers}
Let $X = X_{in} \cup X_{out}$ be a set of points with $z = |X_{out}|$ outliers, satisfying the separation assumption defined before. 
For any $c \geq k+z$, let $T$ be the set returned by GMM$(X,c)$. Then
\[
X_{out} \subset T.
\]
\end{lemma}

\sloppy
Based on Fact~\ref{fact-gmm} and on the above lemma, we can derive
a simple 2-approximation algorithm, which is a slight simplification 
of \cite[Algorithm 3]{AmagataAAAI23}. For a given instance
$(X = X_{in} \cup X_{out},k,z)$ of the max-min diversification with outliers, we first run GMM$(X,k+z)$ to obtain 
a set $T$ of $k+z$ points, together with the distances from their respective nearest neighbors. Then, the subset $T_{out}$ consisting of the $z$
points of $x \in T$ with largest nearest neighbor distance is determined. Finally, the set $T_{in}=T \backslash T_{out}$ is returned as output\footnote{Algorithm 3 of \cite{AmagataAAAI23}, first uses $T$ to compute the set $X_{out}$ of outliers, and then runs GMM again over the set $X \backslash X_{out}=X_{in}$. This further application of GMM is, in fact, not needed, as we prove in Theorem~\ref{thm-2approx}.}. 

\begin{theorem} \label{thm-2approx}
The above algorithm returns a 2-approximation to the max-min diversification problem with outliers on $X= X_{in} \cup X_{out}$. 
\end{theorem}

\begin{proof}
By Lemma~\ref{lem:outliers}, $T$ contains $X_{out}$, therefore, since the nearest neighbor distances can be correctly computed while running GMM, as observed before, we have $T_{out}=X_{out}$ and $T_{in}$ contains exactly $k$ inliers. This implies that 
\[
div(T) \leq div(T_{in}) \leq \ell_k^*.
\]
Let $x^{(k+z)}$ be the last point added to $T$ by GMM, and let $r_{k+z} = d(x^{(k+z)},T \backslash \{x^{(k+z)}\})$. By Fact~\ref{fact-gmm}, we know that $r_{k+z} = div(T)$, and that $X$ can be covered by $k+z-1$ balls of radius $\leq r_{k+z}$, centered at the points of $T \backslash \{x^{(k+z)}\}$. From the separation assumption, we know that each outlier is at distance $\geq \alpha \ell_k^* > \ell_k^*$ from its nearest neighbor, and since $r_{k+z} = div(T) \leq \ell_k^*$, we must have that $z$ of the aforementioned $k+z-1$ balls are centered at the outliers and contain no other point. 
Let $O_k$ be the optimal solution to max-min diversification on $X$. From what we just observed, the points of $O_k$ are distributed among the remaining $k-1$ balls of radius $r_{k+z}$, hence, by the pigeonhole principle, two
of them must belong to the same ball. This implies, that $r_{k+z} \geq \ell_k^*/2$. Putting it all together, we have
\[
div(T_{in}) \geq div(T) = r_{k+z} \geq \ell_k^*/2.
\]
\qed
\end{proof}

We now present a coreset-based algorithm for max-min diversification, based on GMM, which improves upon the one by \cite{AmagataAAAI23}, and can provide very accurate solutions for small $k$ and $z$, when run on low-dimensional datasets. The algorithm, dubbed \emph{CoresetMMD},  is similar in spirit to the ones presented in \cite{CeccarelloVLDB17,CeccarelloPP19,CeccarelloPP20}. For a user defined accuracy parameter $\varepsilon \in (0,1)$, it aims at identifying a $(1+\varepsilon)$-coreset $T$ for the input $X$, such that $X_{out}$ is entirely contained in $T$, and it can be efficiently removed from $T$. Then,
the solution is obtained by running a sequential algorithm ${\cal A}$ on $T_{in}=T \backslash X_{out}$. If $T_{in}$ is small, in particular, if its size is independent of $n=|X|$, one can instantiate ${\cal A}$ with a very accurate (perhaps, exact) yet computationally expensive algorithm.    

More precisely, our algorithm first computes an initial coreset $T$ of size $k+z$, through GMM$(X,k+z)$, 
maintaining, for each $x \in T$, the distance $d(x,\NN(x))$ from its nearest neighbor in $X$. Then, the set $X_{out}$ of outliers is retrieved by taking the $z$ points of $T$ with the largest distance from the respective nearest neighbors. Let $r$ be the radius of the clustering induced by 
$T$. Then, additional iterations of GMM are performed, adding points to $T$ so to build a larger coreset, until the clustering radius drops to a value $\BO{\varepsilon r}$. This will ensure that, at the end, $T$ is a 
$(1+\varepsilon)$-coreset. Finally, a sequential algorithm ${\cal A}$ for max-min diversification is
run on $T \backslash X_{out}$.  (See Algorithm \ref{alg-coreset} for the pseudocode.) 

\medskip

\DontPrintSemicolon
\begin{algorithm}[H]
    \caption{CoresetMMD$(X,k,z,\varepsilon,{\cal A})$}
    \label{alg-coreset}
    $T \gets $ GMM$(X,k+z)\;\;$ (computing also $d(x,\NN(x)), \forall x \in T$)\;
    $X_{out} \gets \{\mbox{$z$ points of $T$ with largest distance from nearest neighbor} \}$\; \label{line:outliers}
    $r \gets \max_{x \in X} d(x,T)$\; \label{line:radius}
    $\varepsilon' \gets \varepsilon/(1+\varepsilon)$ 
    \tcp*{this ensures $1-\varepsilon' = 1/(1+\varepsilon)$}
    \While{$\max_{x \in X} d(x,T) > (\varepsilon'/2)r$}{
    \tcc{run one iteration of GMM}
        $\bar{x} \gets \argmax\limits_{y \in X \setminus T} d(y, T)$\;
        $T \gets T \cup \{\bar{x}\}$\;
    }
    \Return ${\cal A}(T\backslash X_{out},k)$\;
\end{algorithm}

\begin{theorem}
\label{thm:coresetalg}
\sloppy
Let $X = X_{in} \cup X_{out}$ be a dataset with $n$ points, including $z$ outliers, and 
let ${\cal A}$ be a $c$-approximation algorithm for max-min diversification. For any $\varepsilon \in (0,1)$, CoresetMMD$(X,k,z,\varepsilon,{\cal A})$ 
returns a $c(1+\varepsilon)$-approximation to max-min diversification with outliers on $X$. Moreover, if $X$ has doubling dimension $D$, the final coreset $T$ at the end of the while loop has size 
\[
|T| = \BO{k \left(\frac{8 (1+\varepsilon)} {\varepsilon}\right)^D + z},
\]
and the overall running time is 
\[
\BO{|T|n+t_{\cal A}(|T|-z)},
\]
where $t_{\cal A}(\cdot)$ denotes the complexity of ${\cal A}$.
\end{theorem}

\begin{proof}
By repeating the same reasoning used in the proof of Theorem~\ref{thm-2approx}, we can show that the set of outliers 
computed in Line~\ref{line:outliers} is correct, and that the radius $r$ computed in Line~\ref{line:radius} is $\leq \ell_k^*$. 
By construction, at the end of the while loop, each point $x \in X$ has 
a proxy $p(x)$ in the current coreset $T$, with 
\[
d(x,p(x)) \leq (\varepsilon'/2)r \leq (\varepsilon'/2)\ell_k^*,
\]
where $\varepsilon' = \varepsilon/(1+\varepsilon) \in (0,1)$.
Let $O_k \subseteq X_{in}$ be the optimal subset of size $k$ achieving $div(O_k) = \ell_k^*$. Thus, for every $o \in O_k$, we have
$d(o,p(o)) \leq (\varepsilon'/2)\ell_k^*$, and, by the separation assumption, it is easy to see that $p(o)$ cannot be an outlier. Therefore, $p(o)$ is an inlier of the final coreset $T$. By applying \cite[Lemma 1]{CeccarelloVLDB17}, we obtain that $T$ is a $(1+\varepsilon)$-coreset for max-min diversification on $X_{in}$, that is, it is a $(1+\varepsilon)$-coreset for max-min diversification with outliers on $X$. Therefore, the returned solution is a  $c(1+\varepsilon)$-approximation to the optimal solution of the problem.  

We now prove the bound on the size of the final coreset $T$, which will immediately imply the complexity bound as well. 
Let $T^{(k+z)}$ denote the initial coreset returned
by GMM$(X,k+z)$. From the above discussion and the separation assumption we can easily argue that $X$ can be covered by $k+z$ balls of radius $r \leq \ell_k^*$ centered at the points of $T^{(k+z)}$, where the $z$ outliers are
centers of distinct balls containing no inliers. By the property of the doubling dimension, we know that the $k$ balls containing inliers can be in turn covered by $k'$ balls of radius $(\varepsilon'/4)r$, with 
\[
k'= k \left( \frac{8}{\varepsilon'} \right)^D.
\]
This implies that the optimal solution to the $(k'+z)$-center problem on $X$ induces a clustering of radius $\leq (\varepsilon'/4)r$. Since GMM is a 2-approximation algorithm for the $k$-center problem, 
we have that $k'+z$ iterations of GMM are sufficient to yield a clustering radius at most $(\varepsilon'/2)r$. Therefore, the final coreset has size $|T| \leq k'+z$. 
    \qed
\end{proof}

The following corollary is a consequence of the above theorem.
\begin{corollary}
Consider an instance $(X = X_{in} \cup X_{out},k,z)$ of max-min diversification with outliers, where $X$ has $z$ outliers. Suppose that
$k,z \in \BO{1}$, and that $X$ has constant doubling dimension $D$. Then,
for any fixed $\varepsilon \in (0,1)$, a $(1+\varepsilon)$-approximation to
max-min diversification with outliers on $X$ can be computed in linear time. 
\end{corollary}
\begin{proof}
Run CoresetMMD$(X,k,z,\varepsilon,{\cal A})$ instantiating ${\cal A}$ with an exact algorithm. The corollary immediately follows by Theorem~\ref{thm:coresetalg}.    \qed
\end{proof}

\section{Streaming Algorithm} \label{streaming-algorithm}
This section presents a correct algorithm to solve the max-min diversification problem for a (huge) dataset $X$, which is provided in input as a stream (e.g., it is  read from some secondary/tertiary storage system). As before, we assume that $X = X_{in} \cup X_{out}$, with $|X_{out}| = z$, and that 
for some $\alpha>1$ and $\beta>0$, with $\alpha$ suitably larger than $\beta$, we have
\begin{itemize}
\item 
for all $x \in X_{in}$, $d(x,\NN(x)) \leq \beta \ell^*_k$;
\item 
for all $x \in X_{out}$, $d(x,\NN(x)) \geq \alpha \ell^*_k$;
\end{itemize}

In a nutshell, our streaming algorithm constructs a small weighted coreset using very limited working space. We will prove that the  constructed coreset contains the $z$ outliers and enables their identification. At the end of the stream, the weighted coreset is used to extract a high-quality approximate solution for the max-min diversification problem on the entire stream. 

The coreset construction algorithm is similar, in spirit, to the algorithm  presented in \cite{CeccarelloPPS23} for the Matroid Center Problem. The algorithm makes use of two parameters, $\varepsilon', \delta \in (0,1)$ and, for given $k$ and $z$, ensures that after processing the stream $X$, the following information is gathered:
\begin{itemize}
\item
A set $S \subseteq X$ of $k+z$ centers, and a value 
$r \leq (2+\delta) r^*_{k+z}(X)$ such that for every $x \in X$,
$d(x,S) \leq r$;
\item 
A weighted coreset $T \subseteq X$ such that 
for a suitable value $R \in [r/2,r]$, and for $\gamma = 2 (2+\delta)$, the following properties hold:
\begin{description}
\item[(P1)]
For every $a,b \in T$, with $a \neq b$, $d(a,b) > (\varepsilon' / (2 \gamma)) R$;
\item[(P2)]
For every $x \in X$, there exists a proxy $\pi(x) \in T$, such that $d(x,\pi(x)) \leq (\varepsilon' / \gamma) R$ (the proxy, however, is not stored);
\item[(P3)]
For every $y \in T$, the weight $w(y)$ is defined as
\[
w(y) = |\{x \in X \setminus T \; : \; \pi(x)=y\}|.
\]
\end{description}
\end{itemize}

The coreset construction algorithm (see Algorithms 3,4, and 5 for the pseudocode) is divided into 3 phases. The algorithm uses as a subroutine the Scaling Algorithm of \cite{McCutchenK08} for the $k$-center problem, in order to maintain, at each step, a set of $k+z$ ``good'' centers and an estimate $r$ of $r_{k+z}^*(X)$ for the set of stream elements seen so far. The three phases are the following:
\begin{enumerate}
    \item \textbf{Initialization:} this phase executes the main initializations needed to run the streaming algorithm. The starting value of $R$ is set to the value $r$ returned by the Scaling Algorithm on the first $k+z+1$ points of the stream.
    \item \textbf{Insertion:} with this procedure the algorithm decides whether to add the next point $x$ in the stream to the coreset $T$. If $T$ is empty, then $x$ always gets added.
            Otherwise, let $p$ be the closest point to $x$ in $T$. Then, 
            $x$ is added to $T$ iff $x$ and $p$ are sufficiently distant, namely $d(x,p) > (\varepsilon'/(2\gamma))R$,
            else, $x$ is discarded and $w(p)$ is incremented (as $p$ becomes $\pi(x)$). 
    \item \textbf{Shrinking:} this phase is executed only when the value $r$ returned by the Scaling Algorithm grows larger than $2R$.
            If this is the case, then $R$ is set to the current value of $r$, and $T$ is shrunk as follows. Let $T_{old}$ be the coreset $T$ before the shrinkage and set $T = \{u\}$, where $u$ is an arbitrary point of $T_{old}$ (keeping the weight that it had in $T_{old}$); then cycle through the remaining points of $T_{old}$ and, for each such point $y$, add $y$ to $T$ (carrying over its existing weight) if its nearest neighbor in the current $T$ is at distance greater than $(\varepsilon'/(2\gamma))R$.
            Otherwise, letting $y'$ be $y$'s nearest neighbor in the current $T$, ``merge'' $y$ into $y'$ by setting $w(y') \gets w(y') + w(y) + 1$, where the $+1$ accounts for $y$ itself, and discard $y$. A consequence of the merging is that all points processed so far which had $y$ as their proxy in $T_{old}$ (including $y$ itself) will now have $y'$ as their new proxy. (Observe that the coreset $T$ resulting from the shrinking depends on the (arbitrary) order in which the points of $T_{old}$ are processed.)
\end{enumerate}

\DontPrintSemicolon
\begin{algorithm}[H]
\caption{\sc StreamingCoreset($X, k, z, \varepsilon', \delta$)}
$\gamma \gets 2(2+\delta)$\;
\label{line:init1}
$(S,r) \gets$ solution computed by the $\delta$-Scaling Algorithm on the first $k+z+1$ points\;
$R \gets r$; \quad $T \gets \varnothing$\;
\ForEach{$x$ among the first $k+z+1$ points}{
    \textsc{Insert}($x$)
}
\label{line:init2}
\While{the stream is not ended}{
    Let $x$ be the next point of the stream.\;
    Feed $x$ to the Scaling Algorithm, updating $(S,r)$.\;
    \textsc{Insert}($x$)\;
    \If{$r > 2R$}{
        $R \gets r$\; \label{line:updateR}
        \textsc{Shrink}()\; \label{line:shrink}
    }
}
\Return{$(T, w)$}\;
\end{algorithm}

\DontPrintSemicolon
\begin{algorithm}[H]
\caption{\textsc{Insert}($x$)}
\uIf{$T = \varnothing$}{
    $T \gets \{x\}$\; 
    $w(x) \gets 0$\;
}
\Else{
    $p \gets \argmin_{p' \in T} d(x,p')$\; 
    \uIf{$d(x,p) > (\varepsilon'/(2\gamma))R$}{
        $T \gets T \cup \{x\}$; \quad $w(x) \gets 0$\;
    }
    \Else{
        $w(p) \gets w(p) + 1$\ \tcp*{$p$ is the proxy $\pi(x)$}
    }
}
\end{algorithm}

\medskip

\DontPrintSemicolon
\begin{algorithm}[H]
\caption{\textsc{Shrink}()}
$T_{old} \gets T$;\;
Reinitialize $T$ with an arbitrary point $u \in T_{old}$\;
\ForEach{$y \in T_{old} \setminus \{u\}$}{
    $y' \gets \argmin_{y'' \in T} d(y,y'')$\;
    \uIf{$d(y,y') > (\varepsilon'/(2\gamma))R$}{
        $T \gets T \cup \{y\}$ \tcp*{$y$ survives, $w(y)$ unchanged}
    }
    \Else{
        \tcp{$y'$ becomes proxy for $y$ and for all points whose proxy was $y$}
        $w(y') \gets w(y') + w(y) + 1$\; 
    }
}
\end{algorithm}

Before proceeding with the analysis, it is important to observe that 
{\sc StreamingCoreset} is designed in such a way that 
parameter $\delta$, which is inherited from the Scaling Algorithm of \cite{McCutchenK08}, does not affect the stated properties of the returned coreset and, as we will see later, does not influence the approximation ratio of our algorithm. However, it will have a role in regulating the space-time tradeoffs.

The following theorem shows that the coreset construction is correct.
\begin{theorem} \label{thm:correctness}
The coreset $T$ returned by StreamingCoreset$(X,k,z,\varepsilon',\delta)$
satisfies Properties P1, P2, and P3 stated before.
\end{theorem}
\begin{proof}
Property P1 is explicitly enforced by Procedures~{\sc Insert} and
{\sc Shrink}. As for Properties P2 and P3, we argue, by induction, that they hold at the end of each iteration of the while loop of
StreamingCoreset, with respect to the current values of $r,R$, and $T$, and to the prefix of points of the stream processed so far. 
It is easy to see that they hold 
at the beginning of the loop, that is, after the initialization phase
(Lines~\ref{line:init1} -- \ref{line:init2} of StreamingCoreset).
Suppose that they hold when $x$ is the next point to be processed. Clearly, 
they remain valid after the call to {\sc Insert}$(x)$. It remains to show that they also hold in case $r > 2R$ and Lines~\ref{line:updateR} --\ref{line:shrink} are executed. Let $R_{old}$ denote the value of $R$ right before the execution of Line~\ref{line:updateR}, and let $R$ be the new value set by Line~\ref{line:updateR} (thus $R > 2R_{old}$).
Let also $T_{old}$ denote the coreset before the execution of {\sc Shrink}$()$, and $T$ be the coreset resulting after
the execution of {\sc Shrink}$()$.
By induction and by the fact that {\sc Insert}$(x)$ does not create violations to the properties, we have that, before the execution of {\sc Shrink}$()$, for every point $a$ of the stream already processed, there exists a proxy $\pi(a) \in T_{old}$ such that $d(a,\pi(a)) \leq  (\varepsilon' / \gamma) R_{old}$. Now, after the execution of {\sc Shrink}$()$ some element $y \in T_{old}$ might be excluded from $T$, adding  weight $w(y)+1$  to $w(y')$, for some $y' \in T$ with $d(y,y') \leq (\varepsilon' / (2\gamma)) R$. In this case, 
every point $a$ which had $y$ as its proxy in $T_{old}$ (including $y$ itself) will now have $y'$ as new proxy in $T$, and 
\[
d(a,y') \leq d(a,y)+d(y,y') \leq 
\frac{\varepsilon'}{\gamma} R_{old} +
\frac{\varepsilon'}{2\gamma} R
<
\frac{\varepsilon'}{\gamma} R.
\]
This immediately implies that Properties P2 and P3 still hold 
after the execution of {\sc Shrink}$()$.
\qed
\end{proof}

We now argue that if $\varepsilon'$ is sufficiently small, namely
$\varepsilon' < 2(\alpha-\beta)$, 
the set of outliers $X_{out}$ is a subset of $T$ and can be easily
identified, using the weights. 
We first need a technical fact. 
\begin{fact}
$r^*_{k+z}(X) \leq \ell^*_k$.
\end{fact}
\begin{proof}
By \cite[Equation 1]{CeccarelloPP19} we know that  
\[
r^*_{k+z}(X) \leq r^*_k(X_{in}).
\]
The fact follows by the anticover property stated by 
Equation~\ref{eq-anticover}. 
\qed
\end{proof}

The following two lemmas show that $T$ contains all the outliers and 
state useful properties that enable retrieving the outliers from $T$.
\begin{lemma}
\label{lem:outliers2}
If $\alpha > \varepsilon'/2$, then, for every outlier $x \in X_{out}$
we have $x \in T$ and $w(x) = 0$.
\end{lemma}
\begin{proof}
Consider an arbitrary point $x \in X_{out}$. We first show that 
$x \in T$. By contradiction, if this were not the case, the
properties of $T$ would imply the existence
of a proxy $\pi(x) \in T$ such that
\begin{eqnarray*}
d(x,\pi(x)) & \leq & (\varepsilon'/\gamma) R \\
& \leq & (\varepsilon'/(2(2+\delta))) r \\
& \leq & (\varepsilon'/2) r^*_{k+z}(X)  \\
& \leq & (\varepsilon'/2) \ell^*_k < \alpha \ell^*_k.
\end{eqnarray*}
which is impossible since, by the hypotheses on $X_{out}$,
$d(x,\NN(x)) \geq \alpha \ell^*_k$.
Note that the last inequality is justified by the fact that $\alpha > 1 > \varepsilon'/2$.
Finally, observe that $w(x)=0$ since, otherwise, 
there would exist a point $x' \in X \backslash T$ such that
$x = \pi(x')$, which would yield a similar contradiction as the one obtained above.
\qed
\end{proof}

\begin{lemma}
\label{lem:outliersprop}  
If $\alpha > \beta + \varepsilon'/2$, then  $X_{out}$ coincides with the
$z$ 0-weight points of $T$ with largest distance from their respective nearest neighbor in $T$. 
\end{lemma}
\begin{proof}
Let $T' = \{x \in T \; : \; w(x)=0\}$. By Lemma~\ref{lem:outliers2} we know that $X_{out} \subseteq T'$. 
For any outlier $x \in X_{out}$, we have that
$d(x,\NN(x)) \geq \alpha \ell_k^*$, therefore the distance
between $x$ and its nearest neighbor in $T$ is also at least $\alpha \ell_k^*$. Thus, to prove the lemma, it is sufficient to show that for
any inlier $x \in T \cap X_{in}$, with $w(x)=0$, there is a point $y \in T$, with $y \neq x$,
such that $d(x,y) < \alpha \ell^*_k$. Consider an arbitrary inlier 
$x \in T \cap X_{in}$, with $w(x)=0$. 
Let $a = \NN(x)$. If $a \in T$, then, by the hypotheses on $X_{in}$, we know that $d(x,a) \leq \beta \ell^*_k< \alpha \ell^*_k$. If instead $a \not\in T$, by the properties of $T$ we know that there exists $\pi(a) \in T$, such that $d(a,\pi(a)) \leq (\varepsilon'/\gamma) R \leq (\varepsilon'/2) \ell^*_k$. Clearly, since $w(x) = 0$, we must have that $x \neq \pi(a)$, hence
\begin{eqnarray*}
d(x,\pi(a)) & \leq & d(x,a) + d(a,\pi(a)) \\
& \leq & \beta \ell^*_k + (\varepsilon'/2) \ell^*_k \\
& < & \alpha \ell^*_k.
\end{eqnarray*}
\qed
\end{proof}

Once the set of outliers $X_{out}$ has been identified, they can be excluded from $T$, and the solution can be computed by running any standard, sequential max-min diversification algorithm $\cal A$ on $T\setminus X_{out}$.
Algorithm 6 outlines the complete pipeline. Note that the weights are needed only to identify the outliers, while the last step works on the point set $T\setminus X_{out}$, disregarding their weights.

\medskip

\DontPrintSemicolon
\begin{algorithm}[H]
    \caption{Streaming($X, k, z, \varepsilon, \delta, {\cal A}$)}
    \label{alg:complete}
    $\varepsilon' \gets \varepsilon/(1+\varepsilon)$\;
    $(T, w) \gets \textsc{StreamingCoreset}(X, k, z, \varepsilon', \delta)$\;
    Let $(t_1, t_2, \ldots, t_j)$ be the sequence of points $t \in T$ with $w(t) = 0$, sorted in descending order by their nearest neighbor distance in $T$.\;
    $X_{out} \gets \{t_i \mid 1 \le i \le z\}$\;
    \Return{${\cal A}(T \setminus X_{out}, k)$}\;
\end{algorithm}

The following two lemmas analyze the quality of the coreset $T$ computed by \textsc{StreamingCoreset}, and its size.
\begin{lemma}
\label{lem:streaming-coreset}
Let $T$ be the coreset returned by
\textsc{StreamingCoreset}$(X, k, z, \varepsilon', \delta)$.
If $\varepsilon' < 2\alpha$, then $T$  is 
a $1/(1-\varepsilon')$-coreset for max-min diversification with outliers on $X$. 
\end{lemma}
\begin{proof}
Let $O_k \subseteq X_{in}$ be the optimal subset of $k$ inliers, so that $div(O_k) = \ell^*_k$, and, for each $o \in O_k$, consider its proxy $\pi(o) \in T$. Lemma~\ref{lem:outliers2} guarantees that  $\pi(o)$ is an inlier and, by arguing as in the proof of Lemma~\ref{lem:outliers2}, we can show that $d(o,\pi(o)) \leq (\varepsilon'/2) \ell^*_k$. By applying \cite[Lemma 1]{CeccarelloVLDB17} we can then conclude that $T$ is a
$1/(1-\varepsilon')$-coreset for max-min diversification on $X_{in}$, that is, a $1/(1-\varepsilon')$-coreset for max-min diversification with outliers on $X$.
\qed
\end{proof}

\begin{lemma}
\label{lem:coreset-size}
Let $T$ be the coreset returned by $\textsc{StreamingCoreset}(X, k, z, \varepsilon', \delta)$, and define $M$ as the maximum size of $T$ over the course of the algorithm. Then, 
\[
M = O\left((k+z)\left(\frac{32(2+\delta)}{\varepsilon'}\right)^{D}\right)
\]
\end{lemma}
\begin{proof}
Recall that, at any given time in the algorithm, the value $r$ is a $(2+\delta)$-approximation of the optimal radius $r_{k+z}^*(X')$ of a  $(k+z)$-clustering of the set $X' \subseteq X$  of stream points seen so far, and that $R$ is such that $r/2\leq R\leq r$.  Consider now an optimal $(k+z)$-clustering of $X'$ of radius $r_{k+z}^*(X')$. Since the stream has doubling dimension $D$, for any integer $i\geq 0$, each cluster can be partitioned into at most $2^{iD}$ disjoint balls of radius at most $r_{k+z}^*(X')/2^i$.
By choosing $i=\lceil \log_2((r_{k+z}^*(X')/R)(4\gamma/\varepsilon'))\rceil \leq \lceil\log_2(8\gamma/\varepsilon')\rceil$, and recalling that $\gamma=2(2+\delta)$, we obtain that all the elements seen so far can be partitioned into at most $(k+z)(32(2+\delta)/\varepsilon')^D$ disjoint balls of radius at most $(\varepsilon'/(4\gamma))R$. Then, by Property P1,  no such ball can contain two points of $T$, and the lemma follows. 
\qed
\end{proof}

The following theorem states the quality and the performance of our streaming algorithm, and represents the main result of this work.

\begin{theorem}
\label{thm:main}
Let $X = X_{in} \cup X_{out}$ be a dataset of $n$ points in a metric space with doubling dimension $D$, and let $\varepsilon, \delta \in (0,1)$ be two accuracy parameters. Set $\varepsilon' = \varepsilon / (1+\varepsilon)$. Suppose that the outliers satisfy the separation assumption with $\alpha > \beta + \varepsilon' / 2$, and let ${\cal A}$
be a  sequential $\lambda$-approximation algorithm whose
time and space complexities are denoted by $t_{\cal A}(\cdot)$ 
and $s_{\cal A}(\cdot)$, respectively. 
Then, Algorithm~\ref{alg:complete} runs in one pass and returns a $\lambda(1+\varepsilon)$-approximate solution to max-min diversification with outliers for $X$. Moreover, it requires
\begin{itemize}
\item 
a working memory of size
\[
\BO{(k+z)\frac{1}{\delta} \log \frac{1}{\delta} + M + s_{\cal A}(M)};
\]
\item 
a worst-case update time 
\[
\BO{
(k+z)^2 \frac{1}{\delta} \log \left(\frac{1}{\delta}\right) + 
M^2};
\]
\item 
an amortized update time
{\large \[
\begin{array}{ll} 
O\left(\left(\frac{1}{\delta}\right)
\left((k+z)\log \left(\frac{1}{\delta}\right)
+ \frac{\min\{\log \Delta,n\}}{n}(k+z)^2 \right)\right. + & \\[0.3cm]
& \hspace*{-3cm}\left. + M + 
\frac{\min\{\log \Delta, n\}}{n}M^2 \right)
;
\end{array}
\]}
\item
and a query time (to compute the final solution) $\BO{t_{\cal A}(M)}$,
\end{itemize}
where $M = \BO{(k+z)(32(2+\delta)(1+\varepsilon)/\varepsilon)^{D}}$ denotes the maximum 
coreset size and $\Delta$ denotes the aspect ratio of the stream, that is, the ratio between the maximum and minimum pairwise distances between points of the stream.
\end{theorem}
\begin{proof}
Let $\varepsilon' = \varepsilon/(1+\varepsilon)$. Observe that if $\varepsilon \in (0,1)$ then $\varepsilon' \in (0,1)$. This implies that the analysis of \textsc{StreamingCoreset} is valid for $\varepsilon' = \varepsilon/(1+\varepsilon)$. The bound on the approximation ratio follows since, by Lemma~\ref{lem:streaming-coreset}, the coreset $T$ 
returned by $\textsc{StreamingCoreset}(X, k, z, \varepsilon', \delta)$
is a $(1/(1-\varepsilon'))$-coreset, that is a $(1+\varepsilon)$-coreset for max-min diversification with outliers on $X$, and, by Lemma~\ref{lem:outliersprop}, the set given in input to the sequential 
$\lambda$-approximation algorithm ${\cal A}$ consists exactly of all the inliers in $T$. The bound on the working memory size is obtained by adding up 
the space required by the execution of the $\delta$-Scaling Algorithm, which has been 
recalled at the end of Section~\ref{sec:kcenter}, the maximum coreset size, 
bounded from above by Lemma~\ref{lem:coreset-size}, and the space required 
by ${\cal A}$ to compute the final solution. Let us focus now on the update time. The terms dependent on $k$ and $\delta$ are contributed by the
$\delta$-Scaling Algorithm, and have been  recalled at the end of Section~\ref{sec:kcenter}, while the terms dependent on the maximum coreset size $M$ are justified as follows. Consider the arrival of a new point. If Procedure {\sc Shrink}$()$ is not invoked (i.e., $r$ does not grow larger than $2R$), then it is easy to see that the arrival is handled in $\BO{M}$ time, while  if Procedure {\sc Shrink}$()$ is invoked, its execution
requires $\BO{M^2}$ time. However, since every time {\sc Shrink}$()$ is invoked the radius $R$ grows by a factor larger than $2$ and its value is always between $d_{min}/2$ and $(2+\delta)d_{max}$, where $d_{min}$ and $d_{max}$
are, respectively, the minimum and maximum pairwise distances between points of $X$, we have that {\sc Shrink}$()$ is invoked $\BO{\min\{\log \Delta,n\}}$
times over the $n$ arrivals. The bounds on the worst-case and amortized update times follow immediately by the above considerations. 
Finally, the bound on the query time is straightforward.
\qed
\end{proof}
It is important to observe that, when the parameters $\varepsilon$, $\delta$, and $D$ are constant, the working space and query time of our streaming algorithm are linear in $(k+z)$ and do not depend on the stream size $n$, while the worst-case update time is $\BO{(k+z)^2}$ and, again, does not depend on $n$. Finally, the amortized update time is $\BO{(k+z)+(k+z)^2(\log \Delta)/n}$, which can be considerably smaller than the worst-case update time for datasets with reasonable aspect ratios. By multiplying the amortized update time by $n$, we may conclude that the cumulative running time of the streaming algorithm is not dissimilar from the running time of the sequential algorithm presented in Section~\ref{sec-algorithm}, which however requires space linear in $n$.

\bibliographystyle{splncs04}

\end{document}